\documentclass{article}

\usepackage{arxiv}
\usepackage[utf8]{inputenc}
\usepackage[T1]{fontenc}
\usepackage{amsmath,amssymb,amsthm,mathtools}
\usepackage{microtype}
\usepackage[round]{natbib}
\usepackage[colorlinks=true,citecolor=blue,linkcolor=blue,urlcolor=blue]{hyperref}

\newcommand{\Jac}{\texorpdfstring{\ensuremath{\operatorname{Jac}}}{Jac}}
\renewcommand{\headeright}{Preprint}
\renewcommand{\undertitle}{Preprint}
\renewcommand{\shorttitle}{Exponential Convex Calibration Dimension for Jaccard}

\newtheorem{theorem}{Theorem}[section]
\newtheorem{lemma}[theorem]{Lemma}
\newtheorem{corollary}[theorem]{Corollary}
\newtheorem{proposition}[theorem]{Proposition}
\theoremstyle{plain}
\newtheorem{remark}[theorem]{Remark}

\newcommand{\R}{\mathbb{R}}
\newcommand{\cY}{\mathcal{Y}}
\newcommand{\cA}{\mathcal{A}}
\newcommand{\cU}{\mathcal{U}}
\newcommand{\one}{\mathbf{1}}
\DeclareMathOperator{\rank}{rank}

\DeclareMathOperator{\affdim}{affdim}
\DeclareMathOperator{\aff}{aff}
\DeclareMathOperator{\CCdim}{CCdim}
\DeclareMathOperator{\Reg}{Reg}

\title{Exponential Convex Calibration Dimension for the Multi-Label Jaccard Measure}
\author{
  Mingyuan Zhang\\
  Independent Researcher\\
  \href{mailto:myz@alumni.upenn.edu}{myz@alumni.upenn.edu}
}
\date{}

\hypersetup{
  pdftitle={Exponential Convex Calibration Dimension for the Multi-Label Jaccard Measure},
  pdfauthor={Mingyuan Zhang},
  pdfsubject={Machine Learning},
  pdfkeywords={multi-label classification, Jaccard loss, intersection over union, convex calibration dimension, loss-matrix rank, approximate consistency, MinHash random features, F-measure}
}

\begin{document}
\maketitle

\begin{abstract}
The per-instance Jaccard score, or intersection over union (IoU), is standard
in multi-label classification and binary segmentation.  With $s$ labels, its
loss matrix has $2^s$ outcomes and reports.  Under the convention
$\Jac(\varnothing,\varnothing)=1$, we prove that the Jaccard score, shifted-loss,
and ordinary loss matrices are nonsingular and that the loss columns have
affine dimension $2^s-1$.  The proof combines a finite MinHash Gram
representation with Boolean M\"obius inversion.

For exact calibration, we prove
\[
  2^{s-1}
  \le \CCdim(L^{\mathrm{Jac}})
  \le 2^s-1.
\]
The lower bound uses a factorially weighted distribution with
$2^{s-1}+1$ supported outcomes and Bayes-optimal reports.  Consequently, every
exactly calibrated convex surrogate requires exponentially many prediction
coordinates.

We also give two polynomial-dimensional approximation guarantees with explicit
regret transfers.  A new $F_1$-to-Jaccard transfer turns an existing
$(s^2+1)$-dimensional $F_1$ surrogate into a polynomial-time rule with
asymptotic Jaccard regret at most $3-2\sqrt2$.  For any $\alpha>0$ and
$0<\rho<1$, a MinHash square-loss surrogate attains Jaccard-regret floor
$\alpha$ uniformly over arbitrary conditional label distributions.  With
probability at least $1-\rho$, the direct construction has dimension
$O((s^2+s\log(1/\rho))/\alpha^2)$, while a signed variant has dimension
$O((s+\log(1/\rho))/\alpha^2)$.  Thus zero-regret calibration requires
exponential dimension, whereas every fixed additive regret tolerance admits
polynomial prediction dimension.
\end{abstract}


\section{Introduction}

In multi-label classification, an outcome and a prediction are subsets of a
ground set of $s$ labels.  The Jaccard score compares them by the size of their
intersection divided by the size of their union.  The same quantity is widely
called intersection over union (IoU) in image segmentation
\citep{berman2018lovasz}.  Its dependence on the entire predicted and true sets
makes the instance-wise loss nondecomposable across labels.  We study the
decision-theoretic setting in which this per-instance ratio is averaged under
a conditional label distribution, rather than a population-utility setting in
which confusion counts are averaged before a ratio is formed.

The output space contains $2^s$ sets, but an exponential output space alone
does not determine the dimension required by a statistically consistent
convex surrogate.  Some structured losses admit low-dimensional affine
representations or calibrated surrogates even when their report spaces are
exponential.  Convex calibration dimension formalizes the smallest Euclidean
dimension in which an exactly calibrated convex surrogate can exist
\citep{ramaswamy2012classification,ramaswamy2016ccdim}.

Our contributions are threefold.  First, we determine the exact ranks of the
Jaccard score, shifted-loss, and ordinary loss matrices, together with the
affine dimension of the loss columns.  The nonempty score matrix is known to be
strictly positive definite \citep{bouchard2013positive}; we give a
self-contained finite proof tailored to the power set using MinHash and
Boolean M\"obius inversion.  Under our empty-set convention, all three matrix
ranks equal $2^s$, and the affine dimension is $2^s-1$.  The resulting
affine-dimension bound gives an exactly calibrated surrogate in $2^s-1$
dimensions, although matrix rank alone cannot lower-bound arbitrary nonlinear
convex surrogates.

Second, we prove the exponential bounds
\[
  2^{s-1}
  \le \CCdim(L^{\mathrm{Jac}})
  \le 2^s-1.
\]
For the lower bound, fix one core label and assign each outcome containing
$d$ optional labels weight proportional to $1/d!$.  A combinatorial identity
makes all reports containing the core label tie, while adding the core label
strictly improves every nonempty report that omits it.  Mixing with the empty
outcome makes the empty report tie as well.  The supported outcomes and
Bayes-optimal reports then form the same $(2^{s-1}+1)$-element family.  Its
principal score submatrix is nonsingular, making the relevant two-sided
feasible subspace trivial and yielding the lower bound.  Hence the convex
calibration dimension for exact calibration is $\Theta(2^s)$.

Third, we give two complementary polynomial-dimensional approximation
guarantees.  Pointwise, $\Jac=F_1/(2-F_1)$, but expectation does not commute
with this nonlinear transformation, so $F_1$- and Jaccard-optimal reports need
not agree.  Under the same empty-set convention, the multi-label $F_1$ loss
has convex calibration dimension $\Theta(s^2)$ \citep{zhang2026f1rank}, and
quadratic-dimensional convex calibrated surrogates for $F_1$ are known
\citep{nowak2019sharp,zhang2020convex}.  Our regret transfer turns one such
surrogate into a polynomial-time Jaccard rule whose regret is at most
$3-2\sqrt2$ plus a term controlled by its $F_1$ regret.

Our second approximation route uses MinHash directly.  For every $\alpha>0$
and $0<\rho<1$, a finite MinHash feature map uniformly approximates the entire
Jaccard score matrix with probability at least $1-\rho$.  Regressing the
conditional feature mean with a convex square loss gives an
$\alpha$-approximately consistent surrogate: vanishing surrogate regret
guarantees asymptotic Jaccard regret at most $\alpha$.  A signed MinHash
variant has dimension $O((s+\log(1/\rho))/\alpha^2)$.  These are
prediction-dimension guarantees, not efficient-decoding results: the exact
link may still maximize over all $2^s$ reports, while a $\tau$-approximate
decoder increases the regret floor by at most $\tau$.  There is no conflict
with the exact-calibration lower bound, because these constructions allow a
positive tolerance and their dimension diverges as $\alpha\downarrow0$.

\section{Setup and calibration background}

Fix an integer $s\ge1$.  Let $[s]=\{1,\ldots,s\}$,
$\cY=2^{[s]}$, and $N=|\cY|=2^s$.  Outcomes and reports are denoted by
$A,B\in\cY$.  For any finite index set $I$, let
$\one_I\in\R^I$ denote the all-ones vector, and let $\delta_A$ denote the
point mass at $A\in\cY$.  Define the Jaccard score by
\begin{equation}\label{eq:jaccard-score}
  \Jac(A,B)
  =
  \begin{cases}
    \dfrac{|A\cap B|}{|A\cup B|},&A\cup B\ne\varnothing,\\[6pt]
    1,&A=B=\varnothing.
  \end{cases}
\end{equation}
Let $S\in\R^{N\times N}$ be the score matrix,
$S_{A,B}=\Jac(A,B)$, let $U=\one_N\one_N^\top$, and let
\begin{equation}\label{eq:loss-matrix}
  L:=L^{\mathrm{Jac}}=U-S
\end{equation}
be the Jaccard loss matrix.  Thus $L-U=-S$.

For a vector $x$ indexed by a finite set $I$ and $T\subseteq I$, $x_T$
denotes its restriction to $T$.  For a matrix $M$, $M_{T,R}$ denotes the
corresponding submatrix and $M_{\cdot,j}$ its column indexed by $j$.  For a
finite family $V\subseteq\R^r$, let
$\affdim(V)=\dim\aff(V)$.  For a matrix $M=[m_1\ \cdots\ m_k]$, define its
column-affine dimension by
\[
  \affdim(M)=\affdim\{m_1,\ldots,m_k\}.
\]

Let $\Delta_N=\{p\in\R_+^N:\one_N^\top p=1\}$, and write
$\operatorname{supp}(p)=\{A\in\cY:p_A>0\}$ for $p\in\Delta_N$.  The
conditional risk of report $B$ is $p^\top L_{\cdot,B}$, and
\[
  \operatorname{opt}_L(p)
  =\arg\min_{B\in\cY}p^\top L_{\cdot,B}
\]
is the set of Bayes-optimal reports.  Its trigger probability set is
\begin{equation}\label{eq:trigger-set}
  Q_B^L
  =\left\{p\in\Delta_N:
    p^\top L_{\cdot,B}\le p^\top L_{\cdot,B'}
    \text{ for every }B'\in\cY\right\}.
\end{equation}

For completeness, let $\mathcal C\subseteq\R^d$ be convex, let
$\psi:\mathcal C\to\R_+^N$ have convex coordinate functions, and let
$\operatorname{pred}:\mathcal C\to\cY$ be a link.  The pair
$(\psi,\operatorname{pred})$ is $L$-calibrated if, for every
$p\in\Delta_N$,
\[
  \inf_{\substack{u\in\mathcal C:\
       \operatorname{pred}(u)\notin\operatorname{opt}_L(p)}}
       p^\top\psi(u)
  >
  \inf_{u\in\mathcal C}p^\top\psi(u).
\]
The convex calibration dimension $\CCdim(L)$ is the smallest such $d$
\citep[Definitions~1 and~10]{ramaswamy2016ccdim}.

We use two general results.  First,
\begin{equation}\label{eq:ccdim-upper}
  \CCdim(L)\le\affdim(L).
\end{equation}
Second, for $p\in Q_B^L$,
\begin{equation}\label{eq:ccdim-lower-general}
  \CCdim(L)
  \ge \|p\|_0-\mu_{Q_B^L}(p)-1,
\end{equation}
where $\|p\|_0=|\operatorname{supp}(p)|$.  To define $\mu$, for
$Q\subseteq\Delta_N$ and $p\in Q$, let
\[
  \operatorname{dir}_Q(p)
  =\left\{v\in\R^N:
    p+\varepsilon v\in Q
    \text{ for all sufficiently small }\varepsilon>0\right\},
\]
and set
\[
  \operatorname{lin}_Q(p)
  =\operatorname{dir}_Q(p)\cap
    \bigl(-\operatorname{dir}_Q(p)\bigr),
  \qquad
  \mu_Q(p)=\dim\operatorname{lin}_Q(p).
\]
Equations~\eqref{eq:ccdim-upper} and~\eqref{eq:ccdim-lower-general} are
Theorems~12 and~16 of \citet{ramaswamy2016ccdim}.

\section{Related work}

\paragraph{Jaccard matrices and MinHash.}
The collision probability underlying MinHash is the Jaccard similarity
\citep{broder1997resemblance}.  Following an earlier positive-semidefiniteness
result of \citet{gower1971coefficient}, \citet{bouchard2013positive} proved
strict positive definiteness of the complete nonempty Jaccard index matrix by
an infinite positive-semidefinite decomposition.  We give an alternative
finite proof based on MinHash and Boolean M\"obius inversion, then use the
result to analyze the associated loss matrix.  We also subsample the same
Gram representation to obtain a uniformly accurate approximate factorization
and elicit its conditional mean with a convex square loss.

\paragraph{Instance-wise Jaccard prediction.}
\citet{dembczynski2012dependence} studied decision-theoretic risk minimization
for several multi-label losses and highlighted the difficulty of exact
Jaccard-risk minimization from an arbitrary conditional joint distribution.
For an empirical distribution, Bayes-optimal Jaccard prediction is the
Jaccard-median problem; \citet{chierichetti2010median} established its
computational hardness and gave approximation algorithms.  Although Jaccard
is a pointwise monotone transform of the $F_1$ score, the corresponding
expected-score maximizers need not agree \citep{waegeman2014bayes}.
If $\delta(P)$ denotes the Bayes-optimal expected $F_1$ score,
\citet[Theorem~3]{waegeman2014bayes} bound the Jaccard regret of an
$F_1$-optimal report by $1-\delta(P)/2$; this is an approximation guarantee,
not exact Jaccard consistency.  They take
$F_1(\varnothing,\varnothing)=1$ but
$\Jac(\varnothing,\varnothing)=0$, whereas our main convention aligns the two
empty-set values.  Under the aligned convention, we derive a
convention-matched uniform transfer and combine it with the
quadratic-dimensional $F_1$ surrogate of
\citet{zhang2020convex}.

This is the setting considered here: the Jaccard loss is evaluated separately
on each outcome--report pair and then averaged under a conditional
distribution.  It differs from population-confusion or expected-utility
formulations of micro- and instance-averaged linear-fractional metrics, for
which threshold and calibrated-utility methods are available
\citep{koyejo2015consistent,bao2020calibrated}.

\paragraph{Convex surrogates for IoU.}
The Lov\'asz hinge and Lov\'asz--Softmax losses provide tractable objectives
motivated by submodular losses and IoU
\citep{yu2015lovasz,berman2018lovasz}.  These objectives have been influential
in segmentation, but empirical usefulness does not imply calibration for the
finite instance-wise loss studied here.  Indeed,
\citet{finocchiaro2022structured} showed that the Lov\'asz hinge is
inconsistent for its intended structured target unless the underlying set
function is modular.  \citet{dai2023rankseg} derived an IoU-calibrated
ranking-based plug-in rule under conditional independence of the component
labels.  Their exact guarantee relies on conditional independence.  Our exact
lower bound allows arbitrary conditional dependence and concerns all convex
calibrated surrogates, independently of smoothness or polyhedrality.  Our
MinHash construction gives a distribution-free approximate guarantee without
an independence assumption.

\paragraph{Prediction dimension.}
Convex calibration dimension was introduced by
\citet{ramaswamy2012classification} and developed using trigger-set geometry
by \citet{ramaswamy2016ccdim}.  For comparison, under the aligned empty-set
convention, \citet{zhang2026f1rank} proved that the $F_1$ score, shifted-loss,
and ordinary loss matrices have rank $s^2-s+2$, that the loss columns have
affine dimension $s^2-s+1$, and that
$\CCdim(L^{F_1})=\Theta(s^2)$.  Thus exact $F_1$ prediction dimension is
quadratic, whereas the Jaccard bounds proved here are exponential.  The
present exact lower bound applies the feasible-subspace theorem of
\citet{ramaswamy2016ccdim} to a Jaccard-specific family of tied Bayes reports.
We are not aware of a previous exponential prediction-dimension lower bound
for the instance-wise Jaccard loss.  Convex calibration dimension is an exact
notion and does not itself lower-bound the dimension needed for a fixed
additive target-regret tolerance.

\section{Exact rank and affine dimension}

Let $\cY_+=\cY\setminus\{\varnothing\}$ and let
$K=S_{\cY_+,\cY_+}$ be the score matrix restricted to nonempty sets.

\begin{lemma}[Strict positive definiteness]\label{lem:jaccard-pd}
For every $s\ge1$, the matrix $K$ is positive definite.  Consequently,
$\rank(K)=2^s-1$.
\end{lemma}

\noindent
The proof is given in Appendix~\ref{app:proof-pd}.  It writes $K$ as a MinHash
Gram matrix and uses Boolean M\"obius inversion to show that the Gram features
have no nontrivial common null vector.

\begin{theorem}[Exact Jaccard ranks]\label{thm:exact-rank}
For every $s\ge1$, under the convention in~\eqref{eq:jaccard-score},
\[
  \rank(S)=\rank(L-U)=\rank(L)=2^s.
\]
Moreover,
\[
  \affdim(L)=2^s-1.
\]
\end{theorem}

\noindent\emph{Proof sketch.}
Ordering the empty set first gives $S=\operatorname{diag}(1,K)$, so
Lemma~\ref{lem:jaccard-pd} makes $S$ and $L-U=-S$ nonsingular.  Let
$e\in\R^{2^s-1}$ indicate the singleton sets.  For every nonempty $A$,
\[
  (Ke)_A=\sum_{j=1}^s\Jac(A,\{j\})=1,
\]
so $Ke=\one_{N-1}$.  The block kernel equations for $L$ then force every null
vector to vanish, proving that $L$ is nonsingular.  Finally, the $2^s$ score
columns are linearly independent and hence have affine dimension $2^s-1$;
the invertible affine map $x\mapsto\one_N-x$ transfers this dimension to the
loss columns.  The full proof is in Appendix~\ref{app:proof-rank}.

\begin{corollary}[Calibration-dimension upper bound]\label{cor:upper}
For every $s\ge1$,
\[
  \CCdim(L^{\mathrm{Jac}})\le2^s-1.
\]
\end{corollary}

\noindent
This follows immediately from Theorem~\ref{thm:exact-rank} and
\eqref{eq:ccdim-upper}.  Equivalently, one may estimate an arbitrary
$2^s$-class conditional distribution in $2^s-1$ coordinates and decode by
minimizing the estimated Jaccard risk.

\begin{remark}[Alternative empty-set convention]\label{rem:empty-rank}
If instead $\Jac(\varnothing,\varnothing)=0$, then
\[
  \rank(S)=\rank(L-U)=2^s-1,
  \qquad
  \rank(L)=2^s,
  \qquad
  \affdim(L)=2^s-1.
\]
The verification is given in Appendix~\ref{app:empty-convention}.
\end{remark}

\section{An exponential calibration-dimension lower bound}

The lower-bound witness uses a factorial balancing identity.  For an integer
$n\ge0$, write $[n]=\{1,\ldots,n\}$, with $[0]=\varnothing$.

\begin{lemma}[Factorial balancing]\label{lem:factorial-balance}
For every $C\subseteq[n]$,
\begin{equation}\label{eq:factorial-identity}
  \sum_{D\subseteq[n]}
    \frac{1+|C\cap D|}{(1+|C\cup D|)\,|D|!}
  =
  \sum_{d=0}^n\binom nd\frac{1}{(d+1)!}.
\end{equation}
In particular, the left-hand side is independent of $C$.
\end{lemma}

\noindent
The proof is a finite binomial calculation and is given in
Appendix~\ref{app:factorial-balance}.

\begin{theorem}[Exponential CC-dimension bounds]\label{thm:ccdim-bounds}
For every $s\ge1$,
\begin{equation}\label{eq:main-ccdim-bounds}
  2^{s-1}
  \le \CCdim(L^{\mathrm{Jac}})
  \le 2^s-1.
\end{equation}
Consequently,
$\CCdim(L^{\mathrm{Jac}})=\Theta(2^s)$.
\end{theorem}

\noindent\emph{Proof sketch.}
Fix a core label $1$ and put $O=[s]\setminus\{1\}$, so $|O|=s-1$.
Let
\[
  \cU=\{\{1\}\cup D:D\subseteq O\}.
\]
First define a full-support distribution $q$ on $\cU$ by assigning
$\{1\}\cup D$ probability proportional to $1/|D|!$.  For a report
$\{1\}\cup C$, its expected score under $q$ is the normalized left-hand side
of~\eqref{eq:factorial-identity}; hence all $2^{s-1}$ reports in $\cU$ tie.
Every nonempty report missing label $1$ is pointwise improved by adding it,
because all outcomes in $\cU$ contain label $1$.

Let $\kappa>0$ be the common score of the reports in $\cU$.  Mixing $q$ with
the empty outcome using weights $1/(1+\kappa)$ and
$\kappa/(1+\kappa)$, respectively, makes the empty report tie with every
report in $\cU$.  These are exactly the Bayes-optimal reports, and the support
is
\[
  \cA=\{\varnothing\}\cup\cU,
  \qquad |\cA|=2^{s-1}+1.
\]
The active score submatrix indexed by $\cA$ is
$\operatorname{diag}(1,S_{\cU,\cU})$ and is nonsingular by
Lemma~\ref{lem:jaccard-pd}.  Its columns therefore have affine dimension
$|\cA|-1=2^{s-1}$.  At the witness distribution, their difference span is the
entire hyperplane orthogonal to the positive probability vector.  Adding the
normalization constraint leaves no nonzero two-sided feasible direction, so
$\mu=0$.  Substitution into~\eqref{eq:ccdim-lower-general} gives the lower
bound $|\cA|-1=2^{s-1}$.  Corollary~\ref{cor:upper} gives the upper bound.
The detailed feasible-subspace argument is in Appendix~\ref{app:proof-ccdim}.

\begin{remark}[Size of the remaining gap]\label{rem:gap}
The bounds in~\eqref{eq:main-ccdim-bounds} differ by less than a factor of two.
Determining the exact convex calibration dimension, or improving either
constant, remains open.  The theorem rules out every polynomial-dimensional
convex surrogate that is exactly calibrated uniformly over all conditional
label distributions.  Section~\ref{sec:approximate-surrogates} shows that the
qualifier ``exactly'' is essential.
\end{remark}

\begin{remark}[Alternative empty-set convention]\label{rem:empty-ccdim}
Under $\Jac(\varnothing,\varnothing)=0$, the same factorial witness without
the empty-outcome mixture yields
\[
  2^{s-1}-1
  \le \CCdim(L^{\mathrm{Jac}})
  \le 2^s-1.
\]
Thus the exponential conclusion is unchanged.  See
Appendix~\ref{app:empty-convention}.
\end{remark}

\section{Polynomial-dimensional approximate surrogates}
\label{sec:approximate-surrogates}

The exponential lower bound in Theorem~\ref{thm:ccdim-bounds} concerns exact
calibration.  We now allow a nonzero, distribution-free target-regret floor,
combine a prior $F_1$ surrogate with a new regret transfer, and construct a
MinHash surrogate of polynomial prediction dimension.

We use both conditional and population regrets.  For a score
$G:\cY\times\cY\to[0,1]$, a distribution $p\in\Delta_N$, and a report
$B\in\cY$, define
\begin{equation}\label{eq:conditional-score-regret}
  r_G(p,B)
  =
  \max_{C\in\cY}\sum_{A\in\cY}p_A G(A,C)
  -
  \sum_{A\in\cY}p_A G(A,B).
\end{equation}
For a surrogate $\Psi:\cY\times\R^d\to\R_+$, define its conditional regret
by
\begin{equation}\label{eq:conditional-surrogate-regret}
  r_\Psi(p,u)
  =
  \sum_{A\in\cY}p_A\Psi(A,u)
  -
  \inf_{v\in\R^d}\sum_{A\in\cY}p_A\Psi(A,v).
\end{equation}
For a distribution $\mathcal D$ on $\mathcal X\times\cY$, let $p_x$ be the
conditional distribution of the outcome given $X=x$, and set
\begin{align}
  \Reg_G(h)
  &=
  \mathbb E_X\bigl[r_G(p_X,h(X))\bigr],
  \label{eq:population-score-regret}\\
  \Reg_\Psi(f)
  &=
  \mathbb E_{(X,A)\sim\mathcal D}[\Psi(A,f(X))]
  -
  \inf_g
  \mathbb E_{(X,A)\sim\mathcal D}[\Psi(A,g(X))],
  \label{eq:population-surrogate-regret}
\end{align}
where the infimum is over all measurable $g:\mathcal X\to\R^d$.  For
$\alpha\ge0$, we call a surrogate--link pair
\emph{$\alpha$-approximately consistent} for Jaccard if, for every
$\mathcal D$ and every sequence $(f_n)$,
\begin{equation}\label{eq:approximate-consistency-definition}
  \Reg_\Psi(f_n)\longrightarrow0
  \quad\Longrightarrow\quad
  \limsup_{n\to\infty}
  \Reg_{\Jac}(\operatorname{pred}\circ f_n)
  \le\alpha.
\end{equation}
For both surrogates below, the infimum in
\eqref{eq:population-surrogate-regret} decomposes pointwise, and hence
\begin{equation}\label{eq:population-conditional-decomposition}
  \Reg_\Psi(f)=\mathbb E_X\bigl[r_\Psi(p_X,f(X))\bigr].
\end{equation}

\subsection{An \texorpdfstring{$F_1$}{F1}-proxy surrogate}
\label{sec:f1-proxy}

Define the instance-wise $F_1$ score using the same empty-set convention as
in~\eqref{eq:jaccard-score}:
\begin{equation}\label{eq:f1-score-approx}
  F(A,B)
  =
  \begin{cases}
    \dfrac{2|A\cap B|}{|A|+|B|},&|A|+|B|>0,\\[6pt]
    1,&A=B=\varnothing.
  \end{cases}
\end{equation}
Pointwise,
\begin{equation}\label{eq:jaccard-f-transform}
  \Jac(A,B)=g(F(A,B)),
  \qquad
  g(t)=\frac{t}{2-t}.
\end{equation}

\begin{proposition}[\texorpdfstring{$F_1$}{F1}-to-Jaccard regret transfer]
\label{prop:f1-jaccard-transfer}
Let $c_\star=3-2\sqrt2$ and define $H:[0,1]\to[0,1]$ by
\begin{equation}\label{eq:f1-jaccard-transfer-function}
  H(r)
  =
  \begin{cases}
    c_\star+r,&0\le r\le\sqrt2-1,\\[2mm]
    \dfrac{2r}{1+r},&\sqrt2-1\le r\le1.
  \end{cases}
\end{equation}
Then, for every distribution $\mathcal D$ and every classifier
$h:\mathcal X\to\cY$,
\begin{equation}\label{eq:f1-jaccard-global-transfer}
  \Reg_{\Jac}(h)
  \le H\bigl(\Reg_F(h)\bigr)
  \le c_\star+\Reg_F(h).
\end{equation}
In particular, an $F_1$-Bayes classifier has Jaccard regret at most
$c_\star\approx0.1716$.
\end{proposition}

\noindent
The proof is given in Appendix~\ref{app:proof-f1-transfer}.

Let $(\Psi_F,\operatorname{pred}_F)$ denote the
$(s^2+1)$-dimensional convex surrogate and polynomial-time link for $F_1$ of
\citet{zhang2020convex}.  Their calibration and regret-transfer results imply
that, for every distribution $\mathcal D$ and every sequence $(f_n)$,
\[
  \Reg_{\Psi_F}(f_n)\longrightarrow0
  \quad\Longrightarrow\quad
  \Reg_F(\operatorname{pred}_F\circ f_n)\longrightarrow0;
\]
see \citet[Theorems~2 and~5]{zhang2020convex}.  Proposition
\ref{prop:f1-jaccard-transfer} therefore immediately gives
\begin{equation}\label{eq:f1-proxy-jaccard-consistency}
  \limsup_{n\to\infty}
  \Reg_{\Jac}(\operatorname{pred}_F\circ f_n)
  \le c_\star.
\end{equation}
Thus their surrogate--link pair is $c_\star$-approximately consistent for
Jaccard.  We refer to \citet{zhang2020convex} for the construction, decoder,
and proof.  Its quadratic order is necessary for exact $F_1$ calibration
\citep{zhang2026f1rank}; after the transfer above, however, it provides only a
constant-floor approximate guarantee for Jaccard.

Proposition~\ref{prop:f1-jaccard-transfer} is a convention-matched refinement
of \citet[Theorem~3]{waegeman2014bayes}.  For
$\delta(p)=\max_C\mathbb E_p[F(A,C)]$, their coarse comparison gives an
$F_1$-optimal report $B_F$ the bound
$r_{\Jac}(p,B_F)\le1-\delta(p)/2$.  Under the aligned empty-set convention,
the exact identity $\Jac=g(F)$ and Jensen's inequality instead give, for any
report $B$ with $F_1$ regret $r$,
\[
  r_{\Jac}(p,B)\le\delta(p)-g(\delta(p)-r)\le H(r).
\]
In particular,
\[
  r_{\Jac}(p,B_F)
  \le\frac{\delta(p)(1-\delta(p))}{2-\delta(p)}
  \le3-2\sqrt2,
\]
whereas the earlier bound is at least $1/2$.  Thus the proposition both
sharpens the exact-optimizer guarantee and covers approximate $F_1$
optimization; concavity of $H$ supplies the population transfer used above.
Because $F_1$- and Jaccard-optimal reports need not agree, the resulting
guarantee remains approximate rather than exact.

\subsection{A MinHash random-feature surrogate}
\label{sec:minhash-approximate-surrogate}

Write a permutation of $[s]$ as
$\pi=(\pi(1),\ldots,\pi(s))$.  For every nonempty $A\subseteq[s]$, its
MinHash value is the first element of $A$ in this order:
\[
  m_\pi(A)=\pi(k_A),
  \qquad
  k_A=\min\{k\in[s]:\pi(k)\in A\}.
\]
Introduce an additional symbol $\bot$ and extend this map to all
$A\in\cY$ by setting
\[
  \overline m_\pi(A)
  =
  \begin{cases}
    m_\pi(A),&A\ne\varnothing,\\
    \bot,&A=\varnothing,
  \end{cases}
\]
Under the convention $\Jac(\varnothing,\varnothing)=1$, a uniformly random
permutation then satisfies the MinHash identity for every $A,B\in\cY$:
\begin{equation}\label{eq:extended-minhash}
  \Jac(A,B)
  =
  \Pr_\pi\!\left(
    \overline m_\pi(A)=\overline m_\pi(B)
  \right).
\end{equation}

\begin{lemma}[Uniform finite-sample MinHash approximation]
\label{lem:uniform-minhash}
Let $\pi_1,\ldots,\pi_M$ be independent uniformly random permutations of
$[s]$, and let $\{e_j:j\in[s]\cup\{\bot\}\}$ be the standard basis of
$\R^{s+1}$.  Define
\begin{equation}\label{eq:minhash-feature-map}
  \Phi(A)
  =
  \frac{1}{\sqrt M}
  \bigl(
    e_{\overline m_{\pi_1}(A)},\ldots,
    e_{\overline m_{\pi_M}(A)}
  \bigr)
  \in\R^{M(s+1)}
\end{equation}
and
\begin{equation}\label{eq:empirical-minhash-score}
  \widetilde S(A,B)
  =
  \langle\Phi(A),\Phi(B)\rangle
  =
  \frac1M\sum_{r=1}^M
  \one\!\left\{
    \overline m_{\pi_r}(A)=\overline m_{\pi_r}(B)
  \right\}.
\end{equation}
For every $\eta>0$,
\begin{equation}\label{eq:minhash-uniform-probability}
  \Pr\!\left(
    \max_{A,B\in\cY}
    |\widetilde S(A,B)-\Jac(A,B)|>\eta
  \right)
  \le
  2\,4^s\exp(-2M\eta^2).
\end{equation}
Consequently, for $0<\rho<1$, the uniform error is at most $\eta$ with
probability at least $1-\rho$ whenever
\begin{equation}\label{eq:minhash-sample-size}
  M
  \ge
  \frac{2s\log2+\log(2/\rho)}{2\eta^2}.
\end{equation}
\end{lemma}

\noindent
The proof is given in Appendix~\ref{app:proof-minhash-uniform}.

Fix a realization of the feature map and define the convex square-loss
surrogate and link by
\begin{align}
  \Psi_M(A,u)
  &=\|u-\Phi(A)\|_2^2,
  \qquad u\in\R^{M(s+1)},
  \label{eq:minhash-square-surrogate}\\
  \operatorname{pred}_M(u)
  &\in
  \arg\max_{B\in\cY}\langle u,\Phi(B)\rangle,
  \label{eq:minhash-link}
\end{align}
using any fixed rule to break ties.

\begin{theorem}[MinHash approximate consistency]
\label{thm:minhash-approximate-calibration}
On the uniform-approximation event in Lemma~\ref{lem:uniform-minhash}, for
every $p\in\Delta_N$ and every $u\in\R^{M(s+1)}$,
\begin{equation}\label{eq:minhash-conditional-transfer}
  r_{\Jac}(p,\operatorname{pred}_M(u))
  \le
  2\eta+\sqrt{2r_{\Psi_M}(p,u)}.
\end{equation}
The corresponding population guarantee is
\begin{equation}\label{eq:minhash-population-transfer}
  \Reg_{\Jac}(\operatorname{pred}_M\circ f)
  \le
  2\eta+\sqrt{2\Reg_{\Psi_M}(f)}.
\end{equation}
In particular, for every $\alpha>0$, taking $\eta=\alpha/2$ makes the pair
$\alpha$-approximately consistent.  With probability at least $1-\rho$, this
is achieved in dimension
\begin{equation}\label{eq:minhash-alpha-dimension}
  d_{\mathrm{MH}}=M(s+1),
  \qquad
  M=
  \left\lceil
    \frac{2\bigl(2s\log2+\log(2/\rho)\bigr)}{\alpha^2}
  \right\rceil,
\end{equation}
and hence
\[
  d_{\mathrm{MH}}
  =
  O\!\left(
    \frac{s^2+s\log(1/\rho)}{\alpha^2}
  \right).
\]
The probabilistic construction also proves the existence of a fixed
deterministic feature map with this guarantee.
\end{theorem}

\noindent
The proof is given in Appendix~\ref{app:proof-minhash-consistency}.

For the compressed construction, fix an integer $d\ge1$, draw independent
uniform permutations $\pi_1,\ldots,\pi_d$, and, independently of the
permutations and of one another, draw Rademacher variables
$\xi_{r,j}\in\{-1,1\}$ for $r\in[d]$ and
$j\in[s]\cup\{\bot\}$.  Define
\begin{equation}\label{eq:rademacher-feature-map}
  \Phi_\pm(A)
  =
  \frac1{\sqrt d}
  \bigl(
    \xi_{1,\overline m_{\pi_1}(A)},\ldots,
    \xi_{d,\overline m_{\pi_d}(A)}
  \bigr)
  \in\R^d,
\end{equation}
and set
\begin{align}
  \Psi_\pm(A,u)
  &=\|u-\Phi_\pm(A)\|_2^2,
  \label{eq:rademacher-square-surrogate}\\
  \operatorname{pred}_\pm(u)
  &\in\arg\max_{B\in\cY}\langle u,\Phi_\pm(B)\rangle,
  \label{eq:rademacher-link}
\end{align}
using any fixed rule to break ties.

\begin{corollary}[Rademacher-compressed approximate consistency]
\label{cor:rademacher-compression}
For every $\alpha>0$ and $0<\rho<1$, take the preceding random
surrogate--link pair with $d=d_\pm$, where
\begin{equation}\label{eq:rademacher-alpha-dimension}
  d_\pm
  =
  \left\lceil
    \frac{8\bigl(2s\log2+\log(2/\rho)\bigr)}{\alpha^2}
  \right\rceil
\end{equation}
Then, with probability at least $1-\rho$ over its feature construction,
simultaneously for every distribution $\mathcal D$ and every measurable $f$,
\begin{equation}\label{eq:rademacher-population-transfer}
  \Reg_{\Jac}(\operatorname{pred}_\pm\circ f)
  \le
  \alpha+2\sqrt{\Reg_{\Psi_\pm}(f)}.
\end{equation}
Thus the pair is $\alpha$-approximately consistent in dimension
$O((s+\log(1/\rho))/\alpha^2)$.  Fixing $\rho=1/2$ and any successful
realization proves the deterministic existence of such a pair in dimension
$O(s/\alpha^2)$.
\end{corollary}

\noindent
The proof is given in Appendix~\ref{app:proof-rademacher}.

\begin{remark}[Prediction dimension versus decoding]
\label{rem:minhash-decoding}
The MinHash construction has polynomial prediction dimension, but the exact
link in~\eqref{eq:minhash-link} still maximizes over all $2^s$ reports.  No
polynomial-time decoding claim is implicit in
Theorem~\ref{thm:minhash-approximate-calibration}.  If a decoder returns
$\widehat B$ satisfying
\[
  \langle u,\Phi(\widehat B)\rangle
  \ge
  \max_{B\in\cY}\langle u,\Phi(B)\rangle-\tau,
\]
the argument in Appendix~\ref{app:proof-minhash-consistency} adds only $\tau$
to the regret bounds.  The same argument applies to the signed construction.
If a signed-feature decoder returns $\widehat B$ satisfying
\[
  \langle u,\Phi_\pm(\widehat B)\rangle
  \ge
  \max_{B\in\cY}\langle u,\Phi_\pm(B)\rangle-\tau,
\]
then, on the uniform-approximation event used in Corollary
\ref{cor:rademacher-compression},
\[
  r_{\Jac}(p,\widehat B)
  \le
  2\eta+\tau+2\sqrt{r_{\Psi_\pm}(p,u)}.
\]
Consequently, with the choice $\eta=\alpha/2$ in that corollary, the
population bound for the corresponding approximate link
$\widehat{\operatorname{pred}}_\pm$ becomes
\[
  \Reg_{\Jac}(\widehat{\operatorname{pred}}_\pm\circ f)
  \le
  \alpha+\tau+2\sqrt{\Reg_{\Psi_\pm}(f)}.
\]
Thus a fixed decoding error gives $(\alpha+\tau)$-approximate consistency;
to retain an overall floor $\alpha$, choose the kernel-approximation and
decoding budgets so that $2\eta+\tau\le\alpha$.
In contrast, the $F_1$ proxy of
\citet{zhang2020convex} has polynomial-time decoding but, after the transfer in
Proposition~\ref{prop:f1-jaccard-transfer}, retains the nonzero worst-case floor
$c_\star$.
\end{remark}

\begin{remark}[Alternative empty-set convention]
\label{rem:approximate-alternative-empty}
If $\Jac(\varnothing,\varnothing)=0$, define a convention-matched $F_1$ score
with $F(\varnothing,\varnothing)=0$.  Then
\eqref{eq:jaccard-f-transform} and Proposition
\ref{prop:f1-jaccard-transfer} remain valid; any surrogate calibrated for this
convention-matched score inherits the same transfer.  For the MinHash
construction, set $\Phi(\varnothing)=0$ and use the ordinary nonempty-set
features in~\eqref{eq:minhash-feature-map}.
All feature norms are then at most one, and the concentration and regret
bounds remain unchanged; the direct feature dimension becomes $Ms$.  The
signed construction likewise remains valid after setting
$\Phi_\pm(\varnothing)=0$.
\end{remark}

For every fixed $\alpha>0$, the MinHash constructions have polynomial
dimension.  There is no contradiction with
Theorem~\ref{thm:ccdim-bounds}: that theorem concerns the zero-floor case
$\alpha=0$, whereas the dimensions above grow as the requested approximation
floor tends to zero.

\section{Conclusion}

The instance-wise multi-label Jaccard loss has maximal matrix rank and affine
dimension.  A finite MinHash Gram representation and Boolean M\"obius inversion
establish the rank results, while a factorially weighted family with a trivial
two-sided feasible subspace yields
$2^{s-1}\le\CCdim(L^{\mathrm{Jac}})\le2^s-1$.  Thus exact calibration requires
$\Theta(2^s)$ prediction coordinates, although the precise dimension remains
open within a factor of two.

Approximation exposes an exactness--dimension tradeoff.  Our $F_1$ transfer,
combined with the polynomial-time decoder of \citet{zhang2020convex}, gives
asymptotic Jaccard regret at most $3-2\sqrt2$.  For any $\alpha>0$, MinHash
square-loss surrogates attain regret floor $\alpha$ in polynomial dimension;
a signed variant achieves $O((s+\log(1/\rho))/\alpha^2)$ dimensions with
probability at least $1-\rho$.  These prediction-dimension results do not
resolve inference complexity: the MinHash link may still search over $2^s$
reports.  Natural next steps are efficient links with arbitrarily small floors,
matching approximate-dimension lower bounds, and closing the factor-of-two
gap in $\CCdim(L^{\mathrm{Jac}})$.

\appendix

\clearpage
\section{Proofs for the exact-rank results}

\subsection{Proof of Lemma~\ref{lem:jaccard-pd}}\label{app:proof-pd}

\begin{proof}[Proof of Lemma~\ref{lem:jaccard-pd}]
For a permutation $\pi$ of $[s]$ and a nonempty $A\subseteq[s]$, let
$m_\pi(A)$ be the first element of $A$ under $\pi$.  The MinHash identity
\citep{broder1997resemblance} states that, for nonempty $A,B$,
\begin{equation}\label{eq:minhash}
  \Jac(A,B)
  =\Pr_\pi\bigl(m_\pi(A)=m_\pi(B)\bigr).
\end{equation}
Indeed, the first element of $A\cup B$ under a uniformly random permutation is
uniform on $A\cup B$, and the two minima agree exactly when that element lies
in $A\cap B$.

For every permutation $\pi$ and $j\in[s]$, define the column vector
\[
  h_{\pi,j}
  :=\bigl(\mathbf{1}\{m_\pi(A)=j\}\bigr)_{A\in\cY_+}
  \in\R^{\cY_+}\cong\R^{N-1}.
\]
Thus $h_{\pi,j}h_{\pi,j}^\top$ is an $(N-1)\times(N-1)$ matrix indexed by
pairs of nonempty sets.  For every $A,B\in\cY_+$,
\begin{align*}
  \left[
    \frac{1}{s!}\sum_\pi\sum_{j=1}^s
      h_{\pi,j}h_{\pi,j}^\top
  \right]_{A,B}
  &=\frac{1}{s!}\sum_\pi\sum_{j=1}^s
    \mathbf{1}\{m_\pi(A)=j\}\mathbf{1}\{m_\pi(B)=j\}\\
  &=\Pr_\pi\bigl(m_\pi(A)=m_\pi(B)\bigr)
   =\Jac(A,B)=K_{A,B},
\end{align*}
where the sum over $j$ is the indicator that the two MinHashes agree.
Consequently, Equation~\eqref{eq:minhash} gives the Gram representation
\[
  K
  =\frac{1}{s!}\sum_\pi\sum_{j=1}^s
      h_{\pi,j}h_{\pi,j}^\top,
\]
Indeed, for every $x\in\R^{\cY_+}$,
\[
  x^\top Kx
  =\frac{1}{s!}\sum_\pi\sum_{j=1}^s
    x^\top h_{\pi,j}h_{\pi,j}^\top x
  =\frac{1}{s!}\sum_\pi\sum_{j=1}^s
    \bigl(h_{\pi,j}^\top x\bigr)^2
  \ge0,
\]
so $K\succeq0$.  If $x^\top Kx=0$, this finite sum of nonnegative squares is
zero, and therefore $h_{\pi,j}^\top x=0$ for every $\pi$ and $j$.  By the
definition of $h_{\pi,j}$,
\[
  h_{\pi,j}^\top x
  =\sum_{A\in\cY_+}\mathbf{1}\{m_\pi(A)=j\}x_A
  =\sum_{A:m_\pi(A)=j}x_A.
\]
Consequently,
\begin{equation}\label{eq:minhash-null}
  \sum_{A:m_\pi(A)=j}x_A=0
  \qquad\text{for every $\pi$ and $j$}.
\end{equation}

Fix $j$ and any $T\subseteq[s]\setminus\{j\}$.  Choose a permutation in
which every element of $[s]\setminus(T\cup\{j\})$ precedes $j$ and every
element of $T$ follows $j$.  The sets whose first element is $j$ are exactly
$\{j\}\cup R$ with $R\subseteq T$.  Hence~\eqref{eq:minhash-null} becomes
\[
  \sum_{R\subseteq T}x_{\{j\}\cup R}=0
  \qquad\text{for every }T\subseteq[s]\setminus\{j\}.
\]
Boolean M\"obius inversion gives $x_{\{j\}\cup T}=0$ for every $T$.
Every nonempty set contains some $j$, so $x=0$.  Therefore $K\succ0$.
\end{proof}

\subsection{Proof of Theorem~\ref{thm:exact-rank}}\label{app:proof-rank}

\begin{proof}[Proof of Theorem~\ref{thm:exact-rank}]
Order the empty set first.  Since its score with a nonempty set is zero,
\[
  S=
  \begin{pmatrix}
    1&0\\
    0&K
  \end{pmatrix}.
\]
Lemma~\ref{lem:jaccard-pd} makes $S$ nonsingular, so
$\rank(S)=\rank(L-U)=N$.

Let $e\in\R^{N-1}$ be the indicator of the singleton subsets.  For every
nonempty $A$,
\[
  (Ke)_A
  =\sum_{j=1}^s\Jac(A,\{j\})
  =\sum_{j\in A}\frac{1}{|A|}
  =1.
\]
Thus $Ke=\one_{N-1}$ and $\one_{N-1}^\top e=s$.

The loss has block form
\begin{equation}\label{eq:loss-block}
  L=
  \begin{pmatrix}
    0&\one_{N-1}^\top\\
    \one_{N-1}&
      \one_{N-1}\one_{N-1}^\top-K
  \end{pmatrix}.
\end{equation}
Suppose $L(\alpha,x)^\top=0$.  The first block equation gives
$\one_{N-1}^\top x=0$.  The second then reduces to
$Kx=\alpha\one_{N-1}$.  Since $K$ is invertible and $Ke=\one_{N-1}$,
$x=\alpha e$.  Therefore
$0=\one_{N-1}^\top x=\alpha s$, so $\alpha=0$ and $x=0$.  Hence
$\rank(L)=N$.

The $N$ score columns are linearly independent, and therefore affinely
independent, so $\affdim(S)=N-1$.  Applying the invertible affine map
$z\mapsto\one_N-z$ to each score column gives the corresponding loss column.
Thus $\affdim(L)=N-1$.
\end{proof}

\section{Proof of the factorial balancing identity}
\label{app:factorial-balance}

\begin{proof}[Proof of Lemma~\ref{lem:factorial-balance}]
Put $c=|C|$.  Every $D\subseteq[n]$ has the unique decomposition
\[
  D=I\sqcup E,
  \qquad
  I=D\cap C\subseteq C,
  \qquad
  E=D\setminus C\subseteq[n]\setminus C.
\]
Let $i=|I|$ and $j=|E|$.  For this decomposition,
\[
  |C\cap D|=i,
  \qquad
  |C\cup D|=|C|+|E|=c+j,
  \qquad
  |D|=i+j.
\]
For fixed $i$ and $j$, there are $\binom ci$ choices for $I$ and
$\binom{n-c}{j}$ choices for $E$.  Consequently, the left-hand side
of~\eqref{eq:factorial-identity} becomes
\[
  \sum_{j=0}^{n-c}\sum_{i=0}^c
    \binom{n-c}{j}\binom ci
    \frac{i+1}{(c+j+1)(i+j)!}.
\]
Factoring the terms that do not depend on $i$ gives
\begin{equation}\label{eq:factorial-double-sum}
  \sum_{j=0}^{n-c}\binom{n-c}{j}\frac{1}{c+j+1}
  \sum_{i=0}^c\binom ci\frac{i+1}{(i+j)!}.
\end{equation}
We claim that
\begin{equation}\label{eq:factorial-inner}
  \frac{1}{c+j+1}
  \sum_{i=0}^c\binom ci\frac{i+1}{(i+j)!}
  =
  \sum_{r=0}^c\binom cr\frac{1}{(j+r+1)!}.
\end{equation}
To verify the claim, multiply its right-hand side by $c+j+1$ and use
$c+j+1=(j+r+1)+(c-r)$ in each summand.  This gives
\begin{align*}
 &(c+j+1)\sum_{r=0}^c\binom cr\frac{1}{(j+r+1)!} \\
 &\quad=
 \sum_{r=0}^c\binom cr\frac{j+r+1}{(j+r+1)!}
 +\sum_{r=0}^c\binom cr\frac{c-r}{(j+r+1)!} \\
 &\quad=
 \sum_{r=0}^c\binom cr\frac{1}{(j+r)!}
 +\sum_{r=0}^{c-1}\binom cr\frac{c-r}{(j+r+1)!}.
\end{align*}
In the second sum, the $r=c$ term is zero.  Using
\[
  (c-r)\binom cr=(r+1)\binom c{r+1}
\]
and then setting $i=r+1$, we obtain
\[
  \sum_{r=0}^{c-1}\binom cr\frac{c-r}{(j+r+1)!}
  =
  \sum_{i=1}^c i\binom ci\frac{1}{(j+i)!}.
\]
Renaming $r$ as $i$ in the first sum and combining the two sums therefore
yields
\[
  \sum_{i=0}^c\binom ci\frac{1}{(j+i)!}
  +\sum_{i=1}^c i\binom ci\frac{1}{(j+i)!}
  =
  \sum_{i=0}^c\binom ci\frac{i+1}{(j+i)!}.
\]
This is the numerator on the left-hand side
of~\eqref{eq:factorial-inner}; dividing by $c+j+1$ proves the claim.

Substituting~\eqref{eq:factorial-inner} into
\eqref{eq:factorial-double-sum} yields
\[
  \sum_{j=0}^{n-c}\sum_{r=0}^c
    \binom{n-c}{j}\binom cr\frac{1}{(j+r+1)!}.
\]
Set $d=j+r$.  Since $0\le j\le n-c$ and $0\le r\le c$, the new index
runs from $0$ to $n$.  For fixed $d$, the denominator is $(d+1)!$, while
the sum of the binomial coefficients over all admissible pairs
$(j,r)$ is
\[
  \sum_{\substack{0\le r\le c\\0\le d-r\le n-c}}
    \binom cr\binom{n-c}{d-r}
  =\binom nd
\]
by Vandermonde's identity.  Hence the preceding double sum equals
\[
  \sum_{d=0}^n\frac{1}{(d+1)!}
  \sum_{\substack{0\le r\le c\\0\le d-r\le n-c}}
    \binom cr\binom{n-c}{d-r}
  =
  \sum_{d=0}^n\binom nd\frac{1}{(d+1)!},
\]
which is the right-hand side of~\eqref{eq:factorial-identity}.
\end{proof}

\section{Proof of the exponential lower bound}
\label{app:proof-ccdim}

\begin{proof}[Proof of Theorem~\ref{thm:ccdim-bounds}]
The upper bound is Corollary~\ref{cor:upper}.  We prove the lower bound.

\paragraph{Step 1: the factorial distribution.}
Fix the core label $1$, let $O=[s]\setminus\{1\}$ and $n=s-1$, and set
\[
  \cU=\{\{1\}\cup D:D\subseteq O\}.
\]
Define
\[
  Z_n=\sum_{d=0}^n\binom nd\frac{1}{d!},
  \qquad
  H_n=\sum_{d=0}^n\binom nd\frac{1}{(d+1)!},
  \qquad
  \kappa=\frac{H_n}{Z_n}>0.
\]
Let $q$ be supported on $\cU$ with
\begin{equation}\label{eq:q-factorial}
  q_{\{1\}\cup D}=\frac{1}{Z_n|D|!},
  \qquad D\subseteq O.
\end{equation}
Every coordinate in this support is positive.

For $C\subseteq O$, let $A\sim q$ and write $A=\{1\}\cup D$.
Lemma~\ref{lem:factorial-balance} gives
\begin{align*}
  \mathbb E_{A\sim q}\!\left[
    \Jac(A,\{1\}\cup C)
  \right]
  &=\frac1{Z_n}\sum_{D\subseteq O}
    \frac{1+|C\cap D|}{(1+|C\cup D|)|D|!}\\
  &=\frac{H_n}{Z_n}
  =\kappa.
\end{align*}
Thus all reports in $\cU$ tie under $q$.

If $B\ne\varnothing$ and $1\notin B$, then for every outcome
$A\in\cU$,
\begin{equation}\label{eq:core-dominance}
  \Jac(A,B\cup\{1\})-\Jac(A,B)
  =\frac{1}{|A\cup B|}>0.
\end{equation}
Hence $B$ is strictly suboptimal under $q$.  The empty report has score zero,
whereas $\kappa>0$.  Therefore
\begin{equation}\label{eq:q-opt}
  \operatorname{opt}_L(q)=\cU.
\end{equation}

\paragraph{Step 2: tie the empty report.}
Define
\begin{equation}\label{eq:p-witness}
  p
  =\frac{\kappa}{1+\kappa}\,\delta_{\varnothing}
   +\frac{1}{1+\kappa}\,q.
\end{equation}
The empty report has expected score $\kappa/(1+\kappa)$.  Every report in
$\cU$ has the same score because its score on the empty outcome is zero and
its score under $q$ is $\kappa$.  Equations~\eqref{eq:q-opt} and
\eqref{eq:core-dominance} show that all remaining reports have strictly lower
score.  Thus
\begin{equation}\label{eq:p-active}
  \operatorname{opt}_L(p)
  =\cA:=\{\varnothing\}\cup\cU,
  \qquad
  \operatorname{supp}(p)=\cA,
  \qquad
  |\cA|=2^{s-1}+1.
\end{equation}

\paragraph{Step 3: active-column dimension.}
The score submatrix with rows and columns in $\cA$ is
\[
  S_{\cA,\cA}
  =
  \begin{pmatrix}
    1&0\\
    0&S_{\cU,\cU}
  \end{pmatrix}.
\]
The lower-right block is a principal submatrix of the positive-definite matrix
in Lemma~\ref{lem:jaccard-pd}; hence $S_{\cA,\cA}$ is nonsingular.  Its
$|\cA|$ columns are linearly independent and therefore have affine dimension
$|\cA|-1$.

Fix $B_0=\varnothing$ and define the active loss-difference space
\[
  E
  =\operatorname{span}\left\{
    L_{\cA,B}-L_{\cA,B_0}:B\in\cA
  \right\}
  \subseteq\R^{\cA}.
\]
Loss differences are the negatives of score differences, so
\begin{equation}\label{eq:E-dimension}
  \dim E=|\cA|-1=2^{s-1}.
\end{equation}
All reports in $\cA$ tie at $p$, so $p_{\cA}^\top z=0$ for every $z\in E$.
Since $p_{\cA}$ is nonzero and~\eqref{eq:E-dimension} has codimension one,
\begin{equation}\label{eq:E-hyperplane}
  E=p_{\cA}^{\perp}.
\end{equation}

\paragraph{Step 4: the two-sided feasible subspace.}
At $p$, the reports in $\cA$ are exactly tied and every other report has a
strict risk gap.  It follows directly from the simplex constraints and the
active trigger inequalities that
\begin{equation}\label{eq:lineality-characterization}
  \operatorname{lin}_{Q_{B_0}^L}(p)
  =\left\{v\in\R^N:
    v_{\cA^c}=0,\quad
    \one_{\cA}^\top v_{\cA}=0,\quad
    v_{\cA}\perp E
  \right\}.
\end{equation}
Indeed, two-sided feasibility forces zero motion on the zero-probability
coordinates, preserves total mass, and turns every active comparison into an
equality.  Conversely, any vector satisfying these conditions remains in the
simplex and preserves the active equalities for sufficiently small motion in
either direction; the strict inactive gaps remain strict.

By~\eqref{eq:E-hyperplane}, the last condition in
\eqref{eq:lineality-characterization} gives
$v_{\cA}=a p_{\cA}$ for some scalar $a$.  Normalization gives
$0=\one_{\cA}^\top v_{\cA}=a$, because
$\one_{\cA}^\top p_{\cA}=1$.  Thus the lineality space is $\{0\}$ and
\[
  \mu_{Q_{B_0}^L}(p)=0.
\]
Applying~\eqref{eq:ccdim-lower-general} and~\eqref{eq:p-active} gives
\[
  \CCdim(L^{\mathrm{Jac}})
  \ge |\cA|-0-1
  =2^{s-1}.
\]
\end{proof}

\section{Proofs for the polynomial-dimensional approximation results}
\label{app:approximate-surrogate-proofs}

\subsection{Proof of Proposition~\ref{prop:f1-jaccard-transfer}}
\label{app:proof-f1-transfer}

\begin{proof}[Proof of Proposition~\ref{prop:f1-jaccard-transfer}]
Fix $p\in\Delta_N$ and $B\in\cY$.  Let $A$ be a $\cY$-valued random
variable with $\Pr(A=D)=p_D$, and write
\[
  \mathbb E_p[\varphi(A)]
  :=\sum_{D\in\cY}p_D\varphi(D)
\]
for any function $\varphi:\cY\to\R$.  Define the optimal conditional
$F_1$ score and the conditional $F_1$ regret of $B$ by
\[
  \delta=\max_{C\in\cY}\mathbb E_p[F(A,C)],
  \qquad
  r=r_F(p,B),
\]
respectively.  By the definition of $r_F$ in
\eqref{eq:conditional-score-regret},
\[
  \mathbb E_p[F(A,B)]=\delta-r.
\]
Moreover, $0\le r\le\delta\le1$, since $F$ takes values in $[0,1]$.
Because $g(t)\le t$ on $[0,1]$,
\[
  \max_{C\in\cY}\mathbb E_p[\Jac(A,C)]
  \le
  \max_{C\in\cY}\mathbb E_p[F(A,C)]
  =\delta.
\]
The function $g$ is convex.  Jensen's inequality and
\eqref{eq:jaccard-f-transform} therefore give
\[
  \mathbb E_p[\Jac(A,B)]
  =\mathbb E_p[g(F(A,B))]
  \ge g\bigl(\mathbb E_p[F(A,B)]\bigr)
  =g(\delta-r).
\]
Consequently,
\begin{equation}\label{eq:f1-jaccard-conditional-transfer}
  r_{\Jac}(p,B)
  \le \delta-g(\delta-r).
\end{equation}
Set
\[
  z=\delta-r=\mathbb E_p[F(A,B)].
\]
Then $z\ge0$, and $z+r=\delta\le1$, so $z\in[0,1-r]$.  The right-hand
side of~\eqref{eq:f1-jaccard-conditional-transfer} becomes
\[
  r+z-g(z)
  =
  r+\frac{z(1-z)}{2-z}.
\]
Let
\[
  a(z):=z-g(z)=\frac{z(1-z)}{2-z}.
\]
The function $a$ increases on $[0,z_\star]$ and decreases on
$[z_\star,1]$, where $z_\star=2-\sqrt2$ and
$a(z_\star)=c_\star$.  Since the actual value of $z$ lies in
$[0,1-r]$, we have
\[
  r+a(z)
  \le r+\max_{0\le u\le1-r}a(u)
  =H(r),
\]
with $H$ as defined in~\eqref{eq:f1-jaccard-transfer-function}: if
$r\le\sqrt2-1$, the interval contains $z_\star$ and the maximum is
$r+c_\star$; otherwise the maximum occurs at $u=1-r$ and equals
\[
  r+a(1-r)=\frac{2r}{1+r}.
\]
It follows that
\[
  r_{\Jac}(p,B)\le H\bigl(r_F(p,B)\bigr).
\]

It remains to pass from conditional to population regret.  Let $(X,A)$ have
distribution $\mathcal D$, let $p_X$ denote the conditional law of $A$ given
$X$, and let $h$ be the classifier in the proposition.  Define the random
variable
\[
  R_X:=r_F(p_X,h(X))\in[0,1].
\]
The two pieces of $H$ agree in both value and first derivative at
$r=\sqrt2-1$, and the second piece is increasing and concave.  Hence $H$ is
increasing and concave on $[0,1]$.  Applying the conditional inequality with
$p=p_X$ and $B=h(X)$, then using Jensen's inequality and the population-regret
definition~\eqref{eq:population-score-regret}, gives
\begin{align*}
  \Reg_{\Jac}(h)
  &=\mathbb E_X\bigl[r_{\Jac}(p_X,h(X))\bigr]\\
  &\le\mathbb E_X[H(R_X)]\\
  &\le H(\mathbb E_X[R_X])
   =H(\Reg_F(h)).
\end{align*}
Finally, $a(z)\le c_\star$ for every $z\in[0,1]$, so the representation
$H(r)=r+\max_{0\le u\le1-r}a(u)$ gives
$H(r)\le r+c_\star$.  This proves the second inequality in
\eqref{eq:f1-jaccard-global-transfer}.  In particular, if $h$ is
$F_1$-Bayes, then $\Reg_F(h)=0$ and the first inequality gives
$\Reg_{\Jac}(h)\le H(0)=c_\star$.
\end{proof}

\subsection{Proof of Lemma~\ref{lem:uniform-minhash}}
\label{app:proof-minhash-uniform}

\begin{proof}[Proof of Lemma~\ref{lem:uniform-minhash}]
For fixed $A,B$, the summands in
\eqref{eq:empirical-minhash-score} are independent Bernoulli random variables
with mean $\Jac(A,B)$ by~\eqref{eq:extended-minhash}.  Hoeffding's inequality
gives
\[
  \Pr\!\left(
    |\widetilde S(A,B)-\Jac(A,B)|>\eta
  \right)
  \le2\exp(-2M\eta^2).
\]
There are $|\cY|^2=4^s$ ordered outcome--report pairs.  A union bound proves
\eqref{eq:minhash-uniform-probability}, and solving its right-hand side for
$M$ gives~\eqref{eq:minhash-sample-size}.
\end{proof}

\subsection{Proof of Theorem~\ref{thm:minhash-approximate-calibration}}
\label{app:proof-minhash-consistency}

\begin{proof}[Proof of Theorem~\ref{thm:minhash-approximate-calibration}]
Fix a realization for which the uniform-approximation event in
Lemma~\ref{lem:uniform-minhash} holds.
For $p\in\Delta_N$, set
\[
  \mu_p=\sum_{A\in\cY}p_A\Phi(A),
  \qquad
  g_p(B)=\sum_{A\in\cY}p_A\Jac(A,B),
  \qquad
  \widetilde g_p(B)=\langle\mu_p,\Phi(B)\rangle.
\]
By the definitions of $\mu_p$ and $\widetilde S$,
\begin{equation}\label{eq:conditional-empirical-score}
  \widetilde g_p(B)
  =\sum_{A\in\cY}p_A
    \langle\Phi(A),\Phi(B)\rangle
  =\sum_{A\in\cY}p_A\widetilde S(A,B).
\end{equation}
On the uniform-approximation event,
\begin{equation}\label{eq:conditional-score-approximation}
\begin{aligned}
  |g_p(B)-\widetilde g_p(B)|
  &=
  \left|
    \sum_{A\in\cY}p_A
    \bigl(\Jac(A,B)-\widetilde S(A,B)\bigr)
  \right|\\
  &\le
  \sum_{A\in\cY}p_A
  |\Jac(A,B)-\widetilde S(A,B)|\\
  &\le \eta\sum_{A\in\cY}p_A
  =\eta
\end{aligned}
\end{equation}
for every $p$ and $B$.  Let $B^\star\in\arg\max_B g_p(B)$ and
$\widehat B=\operatorname{pred}_M(u)$.  Adding and subtracting the two
approximate scores gives
\begin{align*}
  r_{\Jac}(p,\widehat B)
  &=g_p(B^\star)-g_p(\widehat B)\\
  &=\bigl(g_p(B^\star)-\widetilde g_p(B^\star)\bigr)
    +\bigl(\widetilde g_p(B^\star)
           -\widetilde g_p(\widehat B)\bigr)
    +\bigl(\widetilde g_p(\widehat B)-g_p(\widehat B)\bigr)\\
  &\le 2\eta+\widetilde g_p(B^\star)
                    -\widetilde g_p(\widehat B)\\
  &=2\eta+
    \langle\mu_p,\Phi(B^\star)-\Phi(\widehat B)\rangle\\
  &=
  2\eta+
  \langle\mu_p-u,\Phi(B^\star)-\Phi(\widehat B)\rangle
  +
  \langle u,\Phi(B^\star)-\Phi(\widehat B)\rangle\\
  &\le
  2\eta+
  \|\mu_p-u\|_2
  \|\Phi(B^\star)-\Phi(\widehat B)\|_2.
\end{align*}
Here
\[
  \langle u,\Phi(B^\star)-\Phi(\widehat B)\rangle\le0
\]
because $\widehat B$ maximizes $B\mapsto\langle u,\Phi(B)\rangle$; the final
line then follows from Cauchy--Schwarz.  Moreover, each of the $M$ blocks in
$\Phi(B)$ is a scaled standard basis vector, so
\[
  \|\Phi(B)\|_2^2
  =\frac1M\sum_{r=1}^M\|e_{\overline m_{\pi_r}(B)}\|_2^2
  =1.
\]
All feature coordinates are nonnegative, and hence
$\langle\Phi(B),\Phi(B')\rangle\ge0$.  It follows that
\[
  \|\Phi(B)-\Phi(B')\|_2^2
  =\|\Phi(B)\|_2^2+\|\Phi(B')\|_2^2
   -2\langle\Phi(B),\Phi(B')\rangle
  =2-2\langle\Phi(B),\Phi(B')\rangle
  \le2.
\]
Combining the last two displays yields
\begin{equation}\label{eq:minhash-transfer-distance}
  r_{\Jac}(p,\widehat B)
  \le2\eta+\sqrt2\,\|u-\mu_p\|_2.
\end{equation}

It remains to identify the distance on the right with square-loss regret.
For every $A$,
\[
  u-\Phi(A)=(u-\mu_p)+(\mu_p-\Phi(A)).
\]
After squaring and averaging, the cross term vanishes because
\[
  \sum_Ap_A(\mu_p-\Phi(A))
  =\mu_p-\sum_Ap_A\Phi(A)=0.
\]
Therefore the square-loss bias--variance identity is
\begin{equation}\label{eq:minhash-square-excess}
\begin{aligned}
  \sum_Ap_A\|u-\Phi(A)\|_2^2
  &=\|u-\mu_p\|_2^2
    +\sum_Ap_A\|\Phi(A)-\mu_p\|_2^2.
\end{aligned}
\end{equation}
The second term is independent of $u$, so the conditional surrogate risk is
minimized at $u=\mu_p$, and
\[
  r_{\Psi_M}(p,u)=\|u-\mu_p\|_2^2.
\]
Substitution into~\eqref{eq:minhash-transfer-distance} proves
\eqref{eq:minhash-conditional-transfer}.

Now apply the conditional bound at each $x$ with $p=p_x$ and $u=f(x)$.
Using~\eqref{eq:population-score-regret}, then Cauchy--Schwarz (equivalently,
Jensen's inequality for the square root), gives
\begin{align*}
  \Reg_{\Jac}(\operatorname{pred}_M\circ f)
  &=\mathbb E_X\!\left[
      r_{\Jac}(p_X,\operatorname{pred}_M(f(X)))
    \right]\\
  &\le2\eta+
    \sqrt2\,\mathbb E_X\!\left[
      \sqrt{r_{\Psi_M}(p_X,f(X))}
    \right]\\
  &\le2\eta+
    \sqrt{2\,\mathbb E_X[
      r_{\Psi_M}(p_X,f(X))]}\\
  &=2\eta+\sqrt{2\Reg_{\Psi_M}(f)},
\end{align*}
where the last equality uses
\eqref{eq:population-conditional-decomposition}.  This proves
\eqref{eq:minhash-population-transfer}.

Finally, set $\eta=\alpha/2$.  If
$\Reg_{\Psi_M}(f_n)\to0$, the population bound gives
\[
  \limsup_{n\to\infty}
  \Reg_{\Jac}(\operatorname{pred}_M\circ f_n)
  \le2\eta=\alpha,
\]
which is approximate consistency.  With this value of $\eta$,
\eqref{eq:minhash-sample-size} becomes
\[
  M\ge
  \frac{2\bigl(2s\log2+\log(2/\rho)\bigr)}{\alpha^2}.
\]
Taking the ceiling gives~\eqref{eq:minhash-alpha-dimension}; multiplying by
the block size $s+1$ gives the stated order for $d_{\mathrm{MH}}$.  Since the
uniform-approximation event has positive probability, at least one fixed
realization of the feature map satisfies all these bounds.
\end{proof}

\subsection{Proof of Corollary~\ref{cor:rademacher-compression}}
\label{app:proof-rademacher}

\begin{proof}[Proof of Corollary~\ref{cor:rademacher-compression}]
Take $d=d_\pm$ in~\eqref{eq:rademacher-feature-map}.  For every $A,B$,
$\mathbb E\langle\Phi_\pm(A),\Phi_\pm(B)\rangle=\Jac(A,B)$: conditional on
$\pi_r$, the sign product has expectation one when the two hashes agree and
zero otherwise.  The summands lie in $[-1,1]$, so Hoeffding's inequality and
the same union bound give uniform error $\eta$ with probability at least
$1-\rho$ provided
\[
  d_\pm
  \ge
  \frac{2\bigl(2s\log2+\log(2/\rho)\bigr)}{\eta^2}.
\]
The proof of Theorem~\ref{thm:minhash-approximate-calibration} applies to
\eqref{eq:rademacher-square-surrogate}--\eqref{eq:rademacher-link}, with
$\|\Phi_\pm(B)-\Phi_\pm(B')\|_2\le2$.  Taking $\eta=\alpha/2$ gives
\eqref{eq:rademacher-population-transfer}, while
\eqref{eq:rademacher-alpha-dimension} is exactly the preceding sample-size
condition with this choice of $\eta$.
\end{proof}

\section{The alternative empty-set convention}
\label{app:empty-convention}

Suppose $\Jac(\varnothing,\varnothing)=0$.  Ordering the empty set first,
the score matrix becomes $\operatorname{diag}(0,K)$.  Lemma~\ref{lem:jaccard-pd}
therefore gives
$\rank(S)=\rank(L-U)=N-1$.  The empty score column is zero and the other
$N-1$ columns are independent, so $\affdim(S)=N-1$ and hence
$\affdim(L)=N-1$.

The loss block is now
\[
  L=
  \begin{pmatrix}
    1&\one_{N-1}^\top\\
    \one_{N-1}&
      \one_{N-1}\one_{N-1}^\top-K
  \end{pmatrix}.
\]
If $L(\alpha,x)^\top=0$, the first equation is
$\alpha+\one_{N-1}^\top x=0$.  Substitution into the second gives $Kx=0$, so
$x=0$ and then $\alpha=0$.  Thus $\rank(L)=N$.

For the calibration lower bound, use the factorial distribution $q$ from
\eqref{eq:q-factorial} without mixing in the empty outcome.  Its support and
its exact Bayes-optimal report set are both $\cU$, of size $2^{s-1}$.  The
restricted active score matrix $S_{\cU,\cU}$ is positive definite.  Repeating
Steps 3--4 of the proof of Theorem~\ref{thm:ccdim-bounds} gives zero lineality
and hence the lower bound $2^{s-1}-1$.  The affine-dimension upper bound remains
$2^s-1$.

\end{document}